\documentclass[conference]{IEEEtran}
\IEEEoverridecommandlockouts
\usepackage{amsmath,amssymb,amsthm,graphicx,booktabs,cite,url,microtype}
\usepackage{needspace}
\usepackage{algorithm,algpseudocode}
\usepackage[hidelinks]{hyperref}
\hypersetup{
  pdftitle={FoldQuantVLA: Native Low-Bit Quantization of Vision-Language-Action Models via Consistent Folding},
  pdfauthor={Hung T. Ho, Khanh D. Nguyen, Quang D. Nguyen, Thanh Q. Duong, Ngan Le, Meng Guo, Vien A. Ngo, An T. Le},
  pdfkeywords={Vision-language-action models, post-training quantization, robot deployment}
}
\graphicspath{{figures/}}

\newtheorem{proposition}{Proposition}
\DeclareMathOperator{\diag}{diag}
\DeclareMathOperator{\clip}{clip}
\DeclareMathOperator{\roundop}{round}

\newcommand{\pih}{\ensuremath{\pi_{0.5}}}
\newcommand{\oproj}{\texttt{o\_proj}}
\newcommand{\dproj}{\texttt{down\_proj}}

\begin{document}
\bstctlcite{compactrefs}

\title{\fontsize{20}{24}\selectfont
FoldQuantVLA: Native Low-Bit Quantization of\\
Vision-Language-Action Models via Consistent Folding}
\author{\IEEEauthorblockN{Hung T. Ho\textsuperscript{1}, Khanh D. Nguyen\textsuperscript{1},
Quang D. Nguyen\textsuperscript{1}, Thanh Q. Duong\textsuperscript{1},\\
Ngan Le\textsuperscript{2}, Meng Guo\textsuperscript{3},
Vien A. Ngo\textsuperscript{1,4}, An T. Le\textsuperscript{1,4,5,*}}
\thanks{\textsuperscript{1}VinRobotics, Vietnam.
\textsuperscript{2}AICV Lab, EECS Department, University of Arkansas.
\textsuperscript{3}School of Advanced Manufacturing and Robotics, Peking University, China.
\textsuperscript{4}Center for AI Research, VinUniversity, Vietnam.
\textsuperscript{5}Intelligent Autonomous Systems, TU Darmstadt, Germany.
*Corresponding author: An T. Le (\href{mailto:an@robot-learning.de}{an@robot-learning.de}).}}
\maketitle

\begin{abstract}
Low-bit vision-language-action inference must reduce observation-to-action latency while preserving robot behavior. We present FoldQuantVLA, a post-training quantization framework that carries a consistent activation representation through calibration, weight rounding, and native integer execution. It combines channel scaling and block Hadamard transforms with dynamic per-token quantization, without policy retraining. Custom TensorRT plugins execute projections in both the language backbone and iterative action expert with four-bit weights and activations (W4A4) on Ada GPUs and Jetson AGX Orin. Evaluation spans LIBERO, SimplerEnv, and two robot platforms. Across three GR00T checkpoints and $\pi_{0.5}$, W4A4 achieves $1.20$--$1.33\times$ speedups over floating-point TensorRT on Orin and $1.25$--$1.52\times$ on desktop. Retaining language attention-output and feed-forward down projections at eight bits (W8A8) improves held-out action fidelity on all four checkpoints. Across four real-robot tasks, this configuration raises observed GR00T N1.7 success from $80.0\%$ with uniform W4A4 to $92.5\%$ over 80 trials per configuration, with a measured additional Orin latency of 1\,ms.
\end{abstract}
\section{Introduction}
The vision-language-action (VLA) policies evaluated here combine visual processing, a language backbone, and an iterative denoising action expert~\cite{grootn1,pi05}. Both backbone and expert lie on the latency-critical path. Quantization must also preserve closed-loop behavior: action errors change subsequent observations.
\begin{figure}[!t]
\centering\includegraphics[width=\columnwidth]{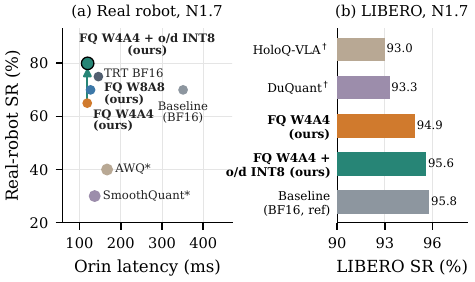}
\caption{FoldQuantVLA at a glance. (a) N1.7 on Orin: observation-to-action latency versus real-robot success rate (SR) on block-stacking task T4, with 20 trials per configuration. Upper left is better: selective INT8 raises observed T4 success from $65\%$ to $80\%$ for 1\,ms. $\ast$ModelOpt SmoothQuant (W8A8) and AWQ (W4A16): T4 only (Table~\ref{tab:robot}). (b) LIBERO SR on N1.7; $\dagger$: our emulated ports (Table~\ref{tab:sota}). Baseline (ref): BF16 PyTorch; FQ: FoldQuantVLA; TRT BF16: floating-point TensorRT; o/d INT8: \oproj{}/\dproj{} at W8A8.}
\label{fig:teaser}
\end{figure}

Post-training quantization reduces precision without retraining. Channel scaling, rotations, and activation-aware rounding address large activations~\cite{smoothquant,quarot,spinquant,duquant,smoothrot,gptq}, but their integration requires consistent calibration and execution coordinates. We assign one change of basis to each shared activation site, fold its inverse into every consuming weight, and round after fixing that basis.

We study uniform W4A4 projections and a selective-INT8 configuration for sensitive language output projections (Fig.~\ref{fig:teaser}). Floating-point TensorRT controls measure gains beyond compilation. Held-out action fidelity screens numerical degradation, while simulator and physical-robot trials assess behavior. Project and code links accompany the paper.\footnote{Project: \url{https://review-artifact-27f4.github.io/foldquantvla/}. Code repository: \url{https://github.com/cair-vinuni/FoldQuantVLA}.}

\needspace{6\baselineskip}
\noindent Our contributions are:
\begin{itemize}
\item \textbf{Consistent folding.} A folding contract that assigns one transformed coordinate system to each shared activation site and carries it through calibration, rounding, and native low-bit execution.
\item \textbf{A W4A4 recipe for VLA policies.} Channel scaling, block Hadamard rotations via the fast Walsh--Hadamard transform (FWHT), and per-token quantization avoid timestep-indexed scale tables. Proposition~\ref{prop:homog} gives conditions for step-invariant normalized projection error.
\item \textbf{Selective INT8 at \oproj{} and \dproj{}.} W8A8 at these sites improves held-out fidelity on all four checkpoints and N1.7 robot success from $80.0\%$ to $92.5\%$ for $1$\,ms on Orin.
\item \textbf{Deployable native low-bit engines.} TensorRT plugins execute W4A4 projections in both backbone and expert, achieving $1.20$--$1.33\times$ speedups over floating-point TensorRT on Orin, with evaluation on LIBERO, SimplerEnv, ALOHA, and SO-101.
\end{itemize}
\section{Related Work}
\textbf{Transforms and rounding.} SmoothQuant scales channels~\cite{smoothquant}; SliceGPT formalizes computational invariance, QuaRot exploits it for quantization, and SpinQuant learns rotations~\cite{slicegpt,quarot,spinquant}. DuQuant, SmoothRot, FlatQuant, and OSTQuant combine or optimize equivalent transforms~\cite{duquant,smoothrot,flatquant,ostquant}. GPTQ approximately optimizes weight rounding~\cite{gptq}, while AWQ protects salient channels~\cite{awq}. Our focus is coordinate consistency across shared consumers and the executable VLA graph.

\textbf{Low-bit execution.} Atom and QServe couple quantization with serving kernels~\cite{atom,qserve}; SVDQuant targets four-bit diffusion~\cite{svdquant}. MixLLM and HAWQ-V3 study mixed precision~\cite{mixllm,hawqv3}. We retain two language projection types at INT8, with integer GEMMs in both configurations at batch~1.

\textbf{VLA quantization.} QuantVLA, QVLA, ActQuant, and DyQ-VLA study calibration, precision allocation, and temporal structure~\cite{quantvla,qvla,actquant,dyqvla}. HoloQ-VLA combines SVD--Hadamard rotations and per-step action-head scales, with native INT4 in both \pih{} backbone and expert~\cite{holoqvla}. Section~\ref{sec:sota} separates our emulated N1.7 recipe comparisons from reported \pih{} results and hardware-specific engine measurements.
\section{Method}

FoldQuantVLA first fixes a shared activation transform and folds its inverse into each consuming weight. It then rounds the folded weights and executes the matching activation transform before integer multiplication. Algorithm~\ref{alg:foldquant} summarizes these offline and runtime stages.

\subsection{Scope and execution model}

A policy maps an observation $o$, instruction $\ell$, and initial denoising noise $\epsilon$ to an action chunk,
\begin{equation}
a=\pi(o,\ell,\epsilon)=\mathcal{D}\bigl(\mathcal{E}(\mathcal{L}(\mathcal{V}(o),\ell),\epsilon)\bigr)\in\mathbb{R}^{H\times d_a},
\label{eq:policy}
\end{equation}
where $\mathcal{V}$ is visual processing, $\mathcal{L}$ the language backbone, $\mathcal{E}$ the iterative action expert, and $\mathcal{D}$ the action decoder. The chunk contains $H$ actions of dimension $d_a$, and proprioception enters $\mathcal{E}$ implicitly. The expert runs $n_\tau$ denoising iterations: four for GR00T, and ten for \pih{}. The controller executes the first $K\le H$ actions before replanning.

We quantize attention and feed-forward projection matrix multiplications (GEMMs) in $\mathcal{L}$ and $\mathcal{E}$. A token is a feature vector; its coordinates are channels. W$b_w$A$b_a$ denotes weight/activation bit widths. Vision, non-projection attention operations, normalization, residual arithmetic, and the action decoder remain floating point. In W4A4 engines, GR00T's adaptive layer-normalization modulation uses weight-only INT4 with floating-point activations. W4A4 and W8A8 projections use per-token activation scales and per-output-channel weight scales. An N1.6 ALOHA audit over 30 observations motivates per-token scaling: at the worst language layer (\dproj{}), activation spread $\max|x|/\mathrm{median}|x|$ is $557{,}518\times$ per tensor versus $8{,}650\times$ per token. The external per-tensor control also differs in calibration and execution (Table~\ref{tab:ablation}). Figure~\ref{fig:fold} illustrates one projection.

\begin{figure*}[!t]
\centering
\includegraphics[width=\textwidth,height=0.25\textheight,keepaspectratio]{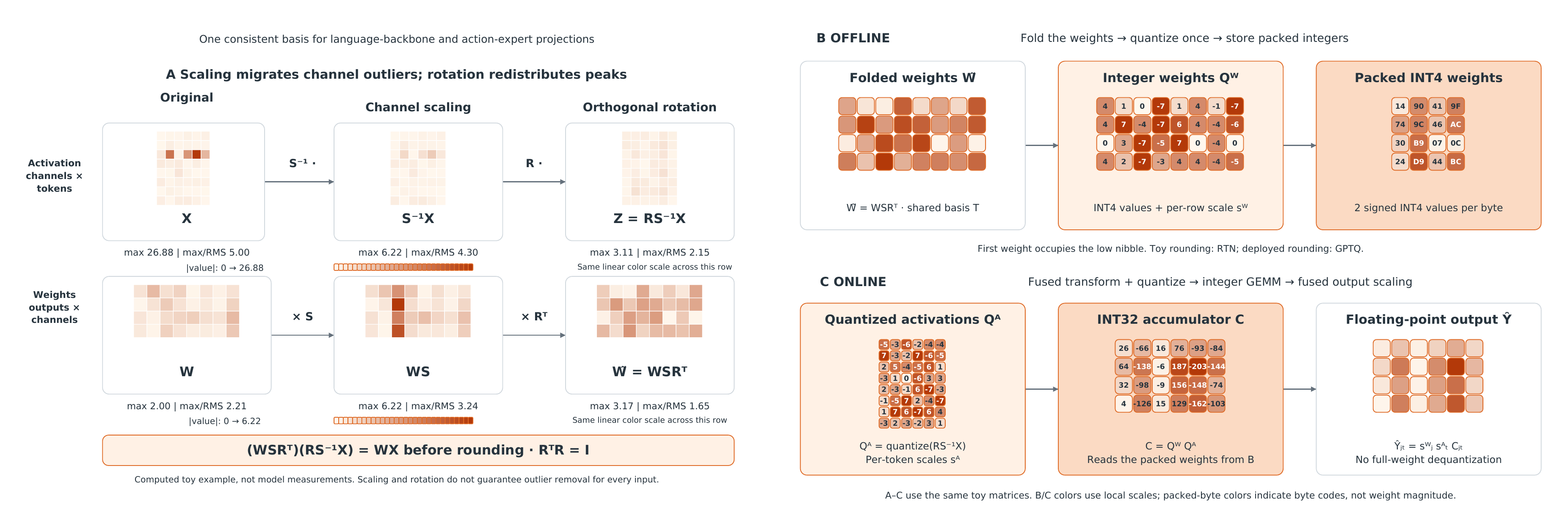}
\caption{Fold offline, execute in low bits (toy matrices, not model measurements). (A) Channel scaling $S^{-1}$ migrates activation outliers into the weights and the orthogonal rotation $R$ mixes the scaled channels; $(WSR^{\mathsf T})(RS^{-1}X)=WX$ before rounding, with activation tokens in the columns of $X$. (B) Offline, the folded weight $\widetilde W=WSR^{\mathsf T}$ is rounded once to INT4 with per-output-channel scales and nibble-packed. (C) Online, the fused prologue quantizes $RS^{-1}X$ per token, the integer GEMM accumulates in INT32, and the output is rescaled by $s^{W}_j\delta_t$ (with $s^A_t=\delta_t$ in the figure) without dequantizing the weights.}
\label{fig:fold}
\end{figure*}

\subsection{A consistent fold}
\label{sec:fold}

For one projection $y=Wx+b$, an invertible transform $T$ gives $z=Tx$ and $\widetilde W=WT^{-1}$, so $\widetilde Wz+b=Wx+b$ before quantization. The same identity must hold for every projection that reuses $z$.

Unroll one policy call into a finite directed acyclic graph with operators $x_v=f_v((x_u)_{u\in\mathrm{pa}(v)})$, where $\mathrm{pa}(v)$ indexes the inputs. Site $v$ holds one activation tensor $x_v$ with $d_v$ channels, shared by its consumers. A fixed invertible $T_v\in\mathbb{R}^{d_v\times d_v}$ acts tokenwise, giving $z_v=T_vx_v$; graph and engine boundaries use identity transforms.

\begin{proposition}[Consistent folding]
\label{prop:fold}
Define $\widetilde f_v((z_u)_{u\in\mathrm{pa}(v)})=T_vf_v((T_u^{-1}z_u)_{u\in\mathrm{pa}(v)})$. The transformed graph computes the same external function in exact arithmetic, as in equivalent-transform methods~\cite{slicegpt,quarot}. An affine operator $x_v=\sum_{u\in\mathrm{pa}(v)}W_{vu}x_u+b_v$ becomes
\begin{equation}
\widetilde W_{vu}=T_vW_{vu}T_u^{-1},\qquad \widetilde b_v=T_vb_v.
\label{eq:fold}
\end{equation}
\end{proposition}

\begin{proof}
In topological order, substituting $z_u=T_ux_u$ gives $z_v=T_vx_v$. Identity output transforms preserve the external function. Substitution into the affine operator gives \eqref{eq:fold}.
\end{proof}

Every projection consuming site $v$ must fold $T_v^{-1}$ into its weights. Query/key/value projections can share one transformed input, as can feed-forward up/gate branches, provided activation bit widths and clipping ratios also match. A consumer using $T_i\ne T_v$ incurs pre-rounding error $W_i(T_i^{-1}T_v-I)x_v$. Such a mismatch lowers N1.6 action-module cosine to $0.9923$ (Table~\ref{tab:ablation}).

\subsection{Transforms supported by the runtime}

A transform must fit the runtime graph. Every transform acts on a projection input and projection outputs keep $T=I$, so residual additions, RoPE, and attention dot products are unchanged; general rotations cannot pass through nonlinearities. We restrict projection-input transforms to
\begin{equation}
T_v=D_v^{\mathrm{o}}R_vD_v^{\mathrm{i}},\qquad R_v^{\mathsf T}R_v=I,
\label{eq:family}
\end{equation}
where $D_v^{\mathrm{i}}$ and $D_v^{\mathrm{o}}$ are positive diagonal matrices. With unchanged output coordinates, a consumer stores $\widetilde W_{iv}=W_{iv}(D_v^{\mathrm{i}})^{-1}R_v^{\mathsf T}(D_v^{\mathrm{o}})^{-1}$.

\emph{Fold-before} uses $(D_v^{\mathrm{i}},D_v^{\mathrm{o}})=(S_v^{-1},I)$, where $S_v=\diag(s_{v,c})$ is calibrated in the original coordinates, giving $\widetilde W=WS_vR_v^{\mathsf T}$. \emph{Fold-after} uses $(I,P_v^{-1})$, where $P_v$ is positive diagonal and calibrated after rotation, giving $\widetilde W=WR_v^{\mathsf T}P_v$. Both preserve the unquantized projection but expose different coordinates to rounding. In general, $R_vS_vR_v^{\mathsf T}$ is dense and cannot be represented by a diagonal $P_v$.

Fold-before is the FWHT default because division fuses into the load loop; fold-after needs an extra pass. At language SmoothQuant sites, let $a_{v,c}$ be the maximum absolute calibration activation in channel $c$, and $w_{v,c}$ the largest input-channel weight magnitude across shared consumers. For $\alpha\in[0,1]$, set $s_{v,c}=a_{v,c}^{\alpha}/w_{v,c}^{1-\alpha}$, with numerical floors and bounded positive scales~\cite{smoothquant}. Statistics use the coordinates where scaling is applied; at $\alpha=1$, unclamped scales equalize calibrated ranges before rotation.

The language path uses the normalized block Hadamard transform $R_v=I_{d_v/\beta}\otimes\beta^{-1/2}\mathcal{H}_\beta$, with $\beta=64$ and $\beta\mid d_v$. Since $\mathcal{H}_\beta\mathcal{H}_\beta^{\mathsf T}=\beta I$, normalization ensures $R_v^{-1}=R_v^{\mathsf T}$. Here $\mathcal{H}_\beta$ is the Sylvester Hadamard matrix with entries $\pm1$, and $\otimes$ repeats it over channel blocks. The FWHT costs $O(d_v\log\beta)$ and preserves the squared norm; it need not reduce the peak of every input. The runtime prologue applies this transform. The action expert uses the same block Hadamard by default, applied as an FWHT, or, optionally, a stored composite rotation (zigzag channel permutation and per-block SVD--Hadamard matrices).

\subsection{Normalization gains}

For a projection following root-mean-square (RMS) normalization, write $x_v=\Gamma_v\bar h_v$, where $\Gamma_v=\diag(\gamma_v)$ contains the learned gain and $\bar h_v$ is normalized without it. Then
\begin{equation}
z_v=D_v^{\mathrm{o}}R_v(D_v^{\mathrm{i}}\Gamma_v)\bar h_v.
\label{eq:norm}
\end{equation}
The product $D_v^{\mathrm{i}}\Gamma_v$ replaces the learned gain without moving a scale through the RMS denominator. This requires a preceding gain, absent at the evaluated \oproj{} and \dproj{} inputs. Other exact folds remain possible: a diagonal scale before \dproj{} can be absorbed into the linear up branch of a gated MLP without crossing the gate nonlinearity.

\subsection{Offline parameters and runtime computation}

Scales, rotations, folded weights, and clipping ratios are fixed at build time. The activation transform remains runtime work. Folding absorbs fixed parameters into neighboring operators; fusion reduces launches and intermediate tensors (Fig.~\ref{fig:arch}).

At a fixed site (index $v$ suppressed), token $z_t=Tx_t$ is quantized to the floating-point reconstruction
\begin{equation}
\widehat z_t=\delta_t\clip\!\left(\roundop(z_t/\delta_t),-q,q\right),\quad \delta_t=\frac{r\lVert z_t\rVert_\infty}{q},
\label{eq:quant}
\end{equation}
where $b_a\in\{4,8\}$, $q=2^{b_a-1}-1$, $r\in(0,1]$ is a clipping ratio calibrated only at language INT4 sites ($r=1$ elsewhere), and rounding and clipping are elementwise. We define $\widehat z_t=0$ when $z_t=0$. Each token supplies its own scale, without a timestep-indexed scale table.

\begin{samepage}
\begin{proposition}[Amplitude invariance]
\label{prop:homog}
In exact arithmetic, \eqref{eq:quant} is positively homogeneous: $\widehat{(\lambda z_t)}=\lambda\widehat z_t$ for $\lambda>0$. For a fixed transform $T$ and fixed quantized folded weights $\widehat W_T$, with the same bias in both projections, the residual $e_T(u)=\widehat W_T\widehat{Tu}-Wu$ therefore satisfies $e_T(\lambda u)=\lambda e_T(u)$. If inputs are almost surely nonzero and $u_\tau/\lVert u_\tau\rVert_2$ follows the same distribution $\nu$ at every denoising step, then
\begin{equation}
\mathcal{L}_\tau(T)=\mathbb{E}\!\left[\frac{\lVert e_T(u_\tau)\rVert_2^2}{\lVert u_\tau\rVert_2^2}\right]
\end{equation}
is independent of $\tau$. If the directional-distribution assumption holds for every candidate in a common feasible set, any minimizer is shared across steps.
\end{proposition}

\begin{proof}
Scaling $z_t$ by $\lambda>0$ scales $\delta_t$ by $\lambda$ and leaves the rounding argument unchanged. This gives both homogeneity identities. Writing $u_\tau=\lVert u_\tau\rVert_2v_\tau$ yields $\mathcal{L}_\tau(T)=\mathbb{E}_{v\sim\nu}\lVert e_T(v)\rVert_2^2$, independent of $\tau$.
\end{proof}
\end{samepage}

One engine accepts the model's timestep conditioning at every iteration without a distributional assumption. Proposition~\ref{prop:homog} additionally requires stable input directions for step-invariant normalized risk. Dynamic scaling removes amplitude variation; it does not address changing channel distributions~\cite{holoqvla} or accumulated error~\cite{dyqvla,accuquant,qdrift}, nor establish optimality for decoded actions or closed-loop behavior.

\begin{figure}[!t]
\centering
\includegraphics[width=\columnwidth]{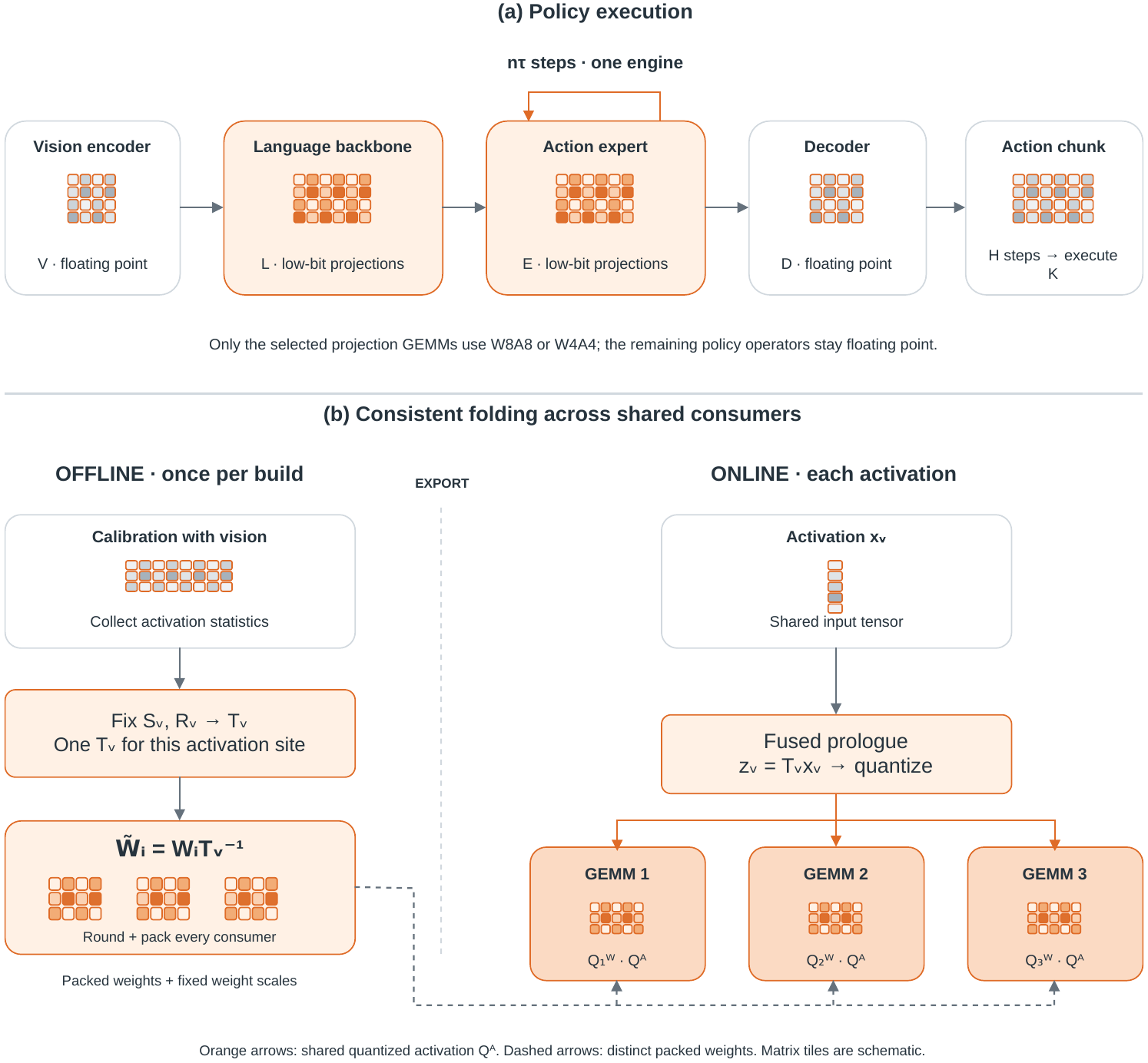}
\caption{(a) Policy execution: W8A8/W4A4 projections in $\mathcal{L}$ and $\mathcal{E}$; one expert engine serves $n_\tau$ steps. The controller executes $K$ of $H$ actions. GR00T's weight-only modulation is omitted. (b) Shared-input execution: offline, fix $T_v$, fold $\widetilde W_i=W_iT_v^{-1}$, round, and pack. Online, transform and quantize once per site; consumers use their own weights and output scales (Algorithm~\ref{alg:foldquant}).}
\label{fig:arch}
\end{figure}

\subsection{Rounding and native execution}

Collect $N_v$ calibration tokens as columns of $U_v\in\mathbb{R}^{d_v\times N_v}$ and set $Z_v=T_vU_v$; one observation may supply many tokens. For four-bit projection weights, GPTQ approximately minimizes reconstruction error in these transformed coordinates:
\begin{align}
\min_{Q\in\mathcal{Q}}\;&\lVert(\widetilde W-Q)Z_v\rVert_F^2, \nonumber\\
G_v&=Z_vZ_v^{\mathsf T}=T_vU_vU_v^{\mathsf T}T_v^{\mathsf T},
\label{eq:gptq}
\end{align}
where $\mathcal{Q}$ contains representable dequantized weight matrices under the chosen scale layout, and $G_v\in\mathbb{R}^{d_v\times d_v}$ is the unnormalized second moment. GPTQ uses $G_v$ with diagonal damping for numerical stability~\cite{gptq}; eight-bit weights use round-to-nearest (RTN). Calibration through the policy forward pass includes visual conditioning. The objective scores weight rounding alone, leaving activation quantization error outside it~\cite{gptaq,qronos}. GPTQ improves fidelity over RTN (Table~\ref{tab:ablation}).

The W4A4 backend uses packed signed-INT4 weights, dynamically quantized signed-INT4 activations, and INT32 accumulation. Write $\widehat W_{jc}=s^W_jq^W_{jc}$ and $\widehat z_{ct}=\delta_tq^A_{ct}$, where $q^W,q^A$ are integer codes, $j$ indexes output channels, and $c$ input channels. The output is
\begin{equation}
\widehat y_{jt}=s^W_j\delta_t\sum_c q^W_{jc}q^A_{ct}+b_j.
\label{eq:integer}
\end{equation}
Thus rescaling and bias addition are floating point; a zero token gives $\widehat y_t=b$. W8A8 uses the corresponding integer path and folded parameters. TensorRT~10.3's built-in INT4 quantization is weight-only~\cite{tensorrt_quant}; our plugins provide the INT4 activation path. The transform and per-token range reduction execute within our plugin, although plugin boundaries can prevent some TensorRT fusions. HadaCore studies related Hadamard kernels~\cite{hadacore}.

Both plugins use CUTLASS~v3.6.0. The four-bit GEMM uses an \texttt{m16n8k64} MMA shape and the eight-bit path uses \texttt{m16n8k32}. Native INT4 runs on Ada and Orin; Section~\ref{sec:threats} documents H100 lowering.

\begin{algorithm}[t]
\caption{FoldQuantVLA at targeted projection sites}
\label{alg:foldquant}
\footnotesize
\begin{algorithmic}[1]
\Require Policy, calibration observations, fold preset, site precisions $p_v\in\{4,8\}$; $\mathcal C(v)$ denotes consumers of site $v$.
\Ensure Packed weights, fixed scales/transforms, and matching runtime operators.
\Statex \textbf{Offline: once per build}
\For{each targeted activation site $v$}
  \State Collect calibration tokens $U_v$ through the policy.
  \State Fix shared $T_v$ from~\eqref{eq:family} and clipping ratio $r_v$ in~\eqref{eq:quant}.
  \State $Z_v\gets T_vU_v$; if $p_v=4$, form $G_v\gets Z_vZ_v^{\mathsf T}$.
  \For{each consumer $i\in\mathcal C(v)$}
    \State $\widetilde W_{iv}\gets W_{iv}T_v^{-1}$.
    \State Round $\widetilde W_{iv}$: GPTQ with $G_v$ at four bits, RTN at eight bits.
    \State Store packed codes $q^W_{iv}$, output-channel scales $s^W_{iv}$, and bias $b_i$.
  \EndFor
\EndFor
\Statex \textbf{Runtime: for each token $x_t$ at site $v$}
\State $z_t\gets T_vx_t$.
\State If $z_t=0$, set $(q^A_t,\delta_t)=(0,0)$; otherwise quantize via~\eqref{eq:quant}.
\For{each consumer $i\in\mathcal C(v)$}
  \State $c_{it}\gets q^W_{iv}q^A_t$ \Comment{INT32 accumulation}
  \State $\widehat y_{it}\gets\delta_t(s^W_{iv}\odot c_{it})+b_i$.
\EndFor
\end{algorithmic}
\end{algorithm}
Algorithm~\ref{alg:foldquant} uses $p_v=b_w=b_a$ for each site and $\odot$ for elementwise multiplication. It shows the logical transform $T_vx_t$: when a scale is absorbed into normalization via~\eqref{eq:norm}, the prologue applies only the remaining factors.

\subsection{Selective INT8 at output projections}
\emph{FQ W4A4} uses uniform low-bit projections. In \emph{FQ W4A4 + o/d INT8}, language attention output (\oproj{}) and feed-forward down (\dproj{}) projections use W8A8; the remaining quantized projections use W4A4. Both configurations retain integer GEMMs at every targeted projection.

These two projection types return features to the residual stream. Their inputs lack the normalization gain used in~\eqref{eq:norm}, and \dproj{} has the largest language activation spread in our audit. These observations motivate the precision choice, but do not prove the cause of sensitivity. The same site choice applies to all four checkpoints, without retraining.

QuantVLA keeps action-expert attention projections in floating point~\cite{quantvla}. BitNet a4.8 uses eight-bit inputs at attention output and feed-forward down projections, with sparsification, in a 1.58-bit model trained for that format~\cite{bitneta48}. Our post-training choice uses W8A8 without sparsification and preserves the folding contract. On N1.6/N1.7, o/d INT8 also uses the dense-rotation calibration preset, so comparisons with the main W4A4 rows include that change; N1.5 and \pih{} change site precision only.

We screen decoded actions against BF16 under shared initial noise, checking coordinate errors and held-out minima because cosine ignores magnitude. Each selected configuration is rebuilt and checked on the assembled engine before closed-loop evaluation.
\section{Experimental Protocol}
We evaluate GR00T N1.7, N1.6, N1.5~\cite{grootn17,grootn16,grootn15}, and \pih{}~\cite{pi05}. N1.7 and N1.5 use our four-suite LIBERO fine-tunes, N1.6 a public four-suite fine-tune, and \pih{} an author-released checkpoint. GR00T uses four denoising steps and action chunks of length 16; the executed prefix is eight actions for N1.7/N1.6 and one for N1.5. For \pih{}, these quantities are ten, ten, and five.

Success tables report percentages (higher is better); bold values mark the best result for each task or checkpoint, including ties. FQ denotes FoldQuantVLA.

\textbf{Fidelity.} P1 checks 32 seeded observations on H100 MIG 3g.40gb against BF16 PyTorch. P2 uses 32 held-out mid-trajectory observations, disjoint from the 128 calibration observations, on RTX~4070~Ti~SUPER. Both compare flattened decoded action chunks under shared initial noise. P2 reports minimum cosine similarity (directional agreement) and median maximum absolute coordinate error (in policy action units) over observations.

\textbf{LIBERO.} P3 uses four suites~\cite{libero}, ten tasks per suite, and 20 initial states per task: 800 episodes per arm, seed~7, 520-step cap, on H100 MIG. Table~\ref{tab:success} compares five deployment arms. Tests use paired per-task success differences with a $t_{39}$ reference; the four preset-matched $p$-values are unadjusted and exceed $0.05$.

\textbf{Latency.} Desktop timing uses one run per checkpoint, 60 observation-to-action iterations after ten warmups, on a dedicated RTX~4070~Ti~SUPER with TensorRT~10.15. Orin uses Jetson AGX Orin with TensorRT~10.3; both use batch~1. Controls include eager PyTorch and floating-point TensorRT. ModelOpt baselines~\cite{nvidia_modelopt} are included on Orin; desktop ModelOpt/compile values use different timing paths and are excluded.

\section{Comparison with Prior W4A4 Recipes}
\label{sec:sota}
\begin{table}[t]
\centering\scriptsize
\setlength{\tabcolsep}{3.4pt}\renewcommand{\arraystretch}{0.95}
\caption{LIBERO success rate (\%) against prior low-bit recipes. Spat./Obj.: Spatial/Object suites; Avg.: mean over four suites. $\Delta$: change from BF16 in percentage points, computed before rounding. Protocols differ between checkpoint blocks.}
\label{tab:sota}
\begin{tabular}{@{}lrrrrrr@{}}
\toprule
Method & Spat. & Obj. & Goal & Long & Avg. & $\Delta$ \\
\midrule
\multicolumn{7}{@{}l}{\emph{GR00T N1.7} (NVIDIA per-suite checkpoints, $720$-step cap)}\\
BF16 PyTorch (reference) & \textbf{98.5} & \textbf{98.5} & 92.5 & \textbf{93.5} & \textbf{95.8} & -- \\
HoloQ-VLA$^{\dag}$ & 94.0 & \textbf{98.5} & 89.5 & 90.0 & 93.0 & $-2.8$ \\
DuQuant$^{\dag}$ & 97.0 & 97.5 & 95.0 & 83.5 & 93.3 & $-2.5$ \\
\cmidrule(l){1-7}
\textbf{FQ W4A4 (ours)} & \textbf{98.5} & 97.0 & \textbf{95.5} & 88.5 & 94.9 & $-0.9$ \\
\textbf{FQ W4A4 + o/d INT8 (ours)} & 96.5 & \textbf{98.5} & 94.0 & \textbf{93.5} & 95.6 & $\mathbf{-0.1}$ \\
\midrule
\multicolumn{7}{@{}l}{\emph{\pih{}} (author-released checkpoint)}\\
BF16 PyTorch (reference) & 98.5 & 99.0 & 97.5 & 93.5 & 97.1 & -- \\
HoloQ-VLA$^{\ast}$ & 99.0 & 97.0 & \textbf{100.0} & \textbf{96.0}$^{\S}$ & \textbf{98.0} & $\mathbf{+0.9}$ \\
DuQuant$^{\ast}$ & 96.0 & 99.0 & 94.0 & 88.0 & 94.3 & $-2.9$ \\
QuantVLA variant$^{\ast}$ & 94.0 & 98.0 & 80.0 & 56.0 & 82.0 & $-15.1$ \\
SmoothQuant$^{\ast}$ & 83.0 & 88.0 & 40.0 & 26.0 & 59.3 & $-37.9$ \\
\cmidrule(l){1-7}
\textbf{FQ W4A4 (ours)} & 98.0 & \textbf{99.5} & 96.0 & 95.5 & 97.3 & $+0.1$ \\
\textbf{FQ W4A4 + o/d INT8 (ours)} & \textbf{99.5} & 99.0 & 98.5 & 94.0 & 97.8 & $+0.6$ \\
\bottomrule
\end{tabular}

\vspace{2pt}
\parbox{\columnwidth}{\scriptsize \emph{N1.7}: $20$ initial states per task, separate from P3; quantized rows share one harness on RTX~4070~Ti~SUPER; BF16 uses the upstream harness on L4. $^{\dag}$Our emulated implementation of the recipe, released with our code. Paired over $40$ tasks, both FoldQuantVLA arms lead both recipes by $13$--$21$ episodes ($p=0.12$--$0.22$). \emph{\pih{}}: BF16 and FoldQuantVLA rows use the released harness ($20$ trials per task, $K=5$); BF16 ($777/800$) matches the rounded $97.1\%$ float rate reported by HoloQ-VLA. $^{\ast}$Reported by HoloQ-VLA~\cite{holoqvla}: full W4A4 including DiT attention, emulated, ten calibration trajectories, 10 trials per task (400 episodes). The QuantVLA row is a full-projection W4A4 variant, not its original selective W4A8 recipe. $^{\S}$Our re-run of their released pack in their harness: $94.0$.}
\end{table}
Table~\ref{tab:sota} compares prior W4A4 recipes on LIBERO. Both FoldQuantVLA configurations score above our N1.7 HoloQ-VLA and DuQuant ports. The emulated ports retain their own projection coverage, so these compare complete recipes; the paired differences are not significant.

For \pih{}, o/d INT8 reaches $97.8\%$, versus HoloQ-VLA's reported $98.0\%$. Our long-suite re-run of its released pack gives $94.0\%$ versus the reported $96.0\%$ ($p=0.44$). HoloQ-VLA also reports native INT4$\times$INT4 execution at $1.03\times$ ($77.08\to75.12$\,ms, RTX~4080~SUPER)~\cite{holoqvla}; FQ W4A4 reaches $2.22\times$ on RTX~4070~Ti~SUPER (Fig.~\ref{fig:latency}). Different hardware and baselines preclude a direct latency comparison.

\section{Native Low-Bit Deployment}
\subsection{Latency beyond the float engine}
\begin{figure*}[t]
\centering\includegraphics[width=0.92\textwidth]{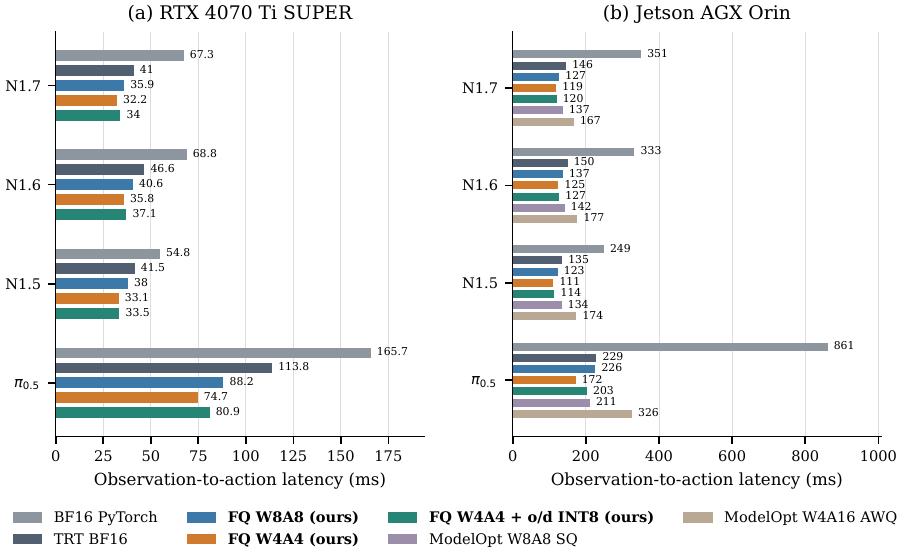}
\caption{Observation-to-action latency at batch~1 (lower is better). Panels share method colors; axes differ across platforms. Bar labels are milliseconds. Each checkpoint is compared with its own floating-point TensorRT engine. W4A4 provides an additional $1.25$--$1.52\times$ speedup on desktop and $1.20$--$1.33\times$ on Orin. Selective o/d INT8 trades some of this gain for action fidelity.}
\label{fig:latency}
\end{figure*}
Figure~\ref{fig:latency} reports observation-to-action latency with W4A4 projections in both backbone and expert. Desktop speedups are $1.66$--$2.22\times$ over eager PyTorch and $1.25$--$1.52\times$ over floating-point TensorRT. W4A4 further reduces latency by $10.3$--$15.3\%$ relative to FQ W8A8.

On Orin, speedups are $2.24$--$5.01\times$ over eager PyTorch and $1.20$--$1.33\times$ over the float engine. Compilation already accounts for $83$--$92\%$ of the eager-to-W4A4 reduction; the remaining gain comes from the low-bit configuration. ModelOpt W4A16 AWQ is slower than the float engine on all four checkpoints: low-bit weight storage alone does not ensure batch-1 acceleration.

\subsection{Cost of selective INT8}
The o/d INT8 configuration adds $1.8$, $1.3$, and $0.4$\,ms to desktop W4A4 for the three GR00T checkpoints, and $1$, $2$, and $3$\,ms on Orin. Thus the N1.7 configuration used on the robot takes $120$\,ms rather than $119$\,ms, while remaining faster than W8A8 at $127$\,ms and float TensorRT at $146$\,ms. The cost is larger for \pih{}: $6.2$\,ms on desktop and $31$\,ms on Orin. Selective INT8 is therefore a checkpoint-dependent trade-off.

\subsection{Storage and resident memory}
For N1.6, weight storage falls from $3802$\,MB for floating-point TensorRT to $2104$\,MB for W8A8 and $1011$\,MB for W4A4; o/d INT8 uses $1189$\,MB. Complete serialized TensorRT engines occupy $5325$ and $2525$\,MB for floating point and W4A4, respectively. Resident memory for float/W4A4 is $6349/4673$\,MiB when replaced PyTorch weights are never materialized, and $10609/8933$\,MiB on the shipped serving path.

\begin{figure*}[!t]
\centering\includegraphics[width=0.82\textwidth]{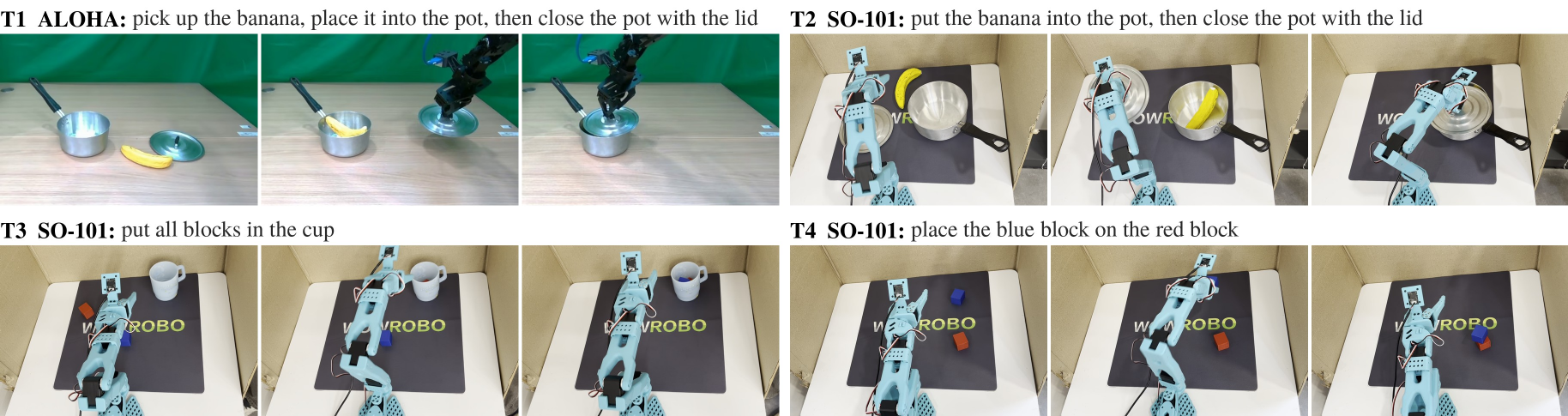}
\caption{Real-robot tasks, with successive frames from one successful episode per task. T1 uses single-arm ALOHA; T2--T4 use SO-101. T1/T2 require the banana in the pot and the lid placed on it; T2 does not require the lid to sit tightly. T3 requires every block inside the cup, and T4 requires the blue block resting on the red block. Frames illustrate tasks, not an arm-to-arm comparison.}
\label{fig:real_tasks}
\end{figure*}
\section{Accuracy and Selective Precision}
\subsection{Main closed-loop comparison}
\begin{table}[t]
\centering\scriptsize
\setlength{\tabcolsep}{4pt}\renewcommand{\arraystretch}{1.0}
\caption{LIBERO P3 success (\%); 800 episodes per cell, using the deployment configurations. Preset-matched selective-precision comparisons are reported separately in the text.}
\label{tab:success}
\begin{tabular*}{\columnwidth}{@{\extracolsep{\fill}}lrrrr@{}}
\toprule
Method & N1.7 & N1.6 & N1.5 & $\pi_{0.5}$\\\midrule
BF16 PyTorch & \textbf{96.25} & 96.38 & 86.38 & 96.50 \\
TRT BF16 & 95.50 & \textbf{97.75} & 86.00 & \textbf{98.00} \\
\midrule
\textbf{FQ W8A8 (ours)} & 95.38 & 95.75 & 87.12 & 97.38 \\
\textbf{FQ W4A4 (ours)} & 95.38 & 95.75 & \textbf{87.38} & 97.12 \\
\textbf{FQ W4A4 + o/d INT8 (ours)} & 95.00 & 96.62 & 87.00 & 97.62 \\
\bottomrule
\end{tabular*}
\end{table}
Table~\ref{tab:success} reports closed-loop behavior for the five deployment configurations.

To isolate the precision change in the o/d INT8 comparison, we use controls with the same base preset and calibration rather than the main W4A4 row. N1.6 improves from $763$ to $773/800$ ($p=0.17$); N1.7 from $757$ to $760$ ($p=0.77$). N1.5 changes from $699$ to $696$ ($p=0.78$), and \pih{} from $777$ to $781$ ($p=0.42$). These separate paired tests do not resolve a LIBERO gain from selective INT8. We therefore motivate it through fidelity and real-robot evidence, without claiming simulator superiority.

\subsection{SimplerEnv Bridge}
\begin{table}[t]
\centering\scriptsize
\setlength{\tabcolsep}{4pt}\renewcommand{\arraystretch}{1.0}
\caption{SimplerEnv Bridge (WidowX) with GR00T N1.6: SR (\%) over 200 episodes per task on RTX~4070~Ti~SUPER. FQ denotes our method.}
\label{tab:simpler}
\begin{tabular*}{\columnwidth}{@{\extracolsep{\fill}}lrrrr@{}}
\toprule
Task & BF16 & \shortstack[r]{\textbf{FQ W8A8}\\\textbf{(ours)}} & \shortstack[r]{\textbf{FQ W4A4}\\\textbf{(ours)}} & \shortstack[r]{\textbf{FQ W4A4 + o/d}\\\textbf{INT8 (ours)}}\\\midrule
Spoon on towel & 63.5 & 63.5 & 66.0 & \textbf{75.5} \\
Carrot on plate & 60.5 & 59.5 & \textbf{70.5} & 66.5 \\
Eggplant in basket & \textbf{92.5} & 89.5 & 58.0 & 70.5 \\
Stack cube & 6.0 & 2.5 & \textbf{7.5} & 5.5 \\
Eggplant in sink & 43.0 & 40.0 & \textbf{63.0} & 50.0 \\
Close drawer & 72.0 & 64.5 & 82.0 & \textbf{89.5} \\
Open drawer & \textbf{98.0} & 96.5 & 81.5 & 92.0 \\
\midrule
Mean & 62.2 & 59.4 & 61.2 & \textbf{64.2} \\
\bottomrule
\end{tabular*}
\end{table}
We also evaluate the released GR00T N1.6 Bridge checkpoint on seven WidowX tasks from the GR00T fork of SimplerEnv~\cite{simplerenv,simplerfork} (Table~\ref{tab:simpler}; 300-step cap, four executed actions). Mean success is similar in these runs: $61.2\%$ for W4A4 and $64.2\%$ for o/d INT8, against $62.2\%$ for BF16 (paired over tasks, $p=0.89$ and $0.70$). These tests do not establish equivalence. The aggregate masks larger per-task differences: the mean absolute per-task difference between W4A4 and BF16 is $13.6$ percentage points, with its largest loss on the basket task ($-34.5$). Selective INT8 reduces it to $10.1$ points and improves open drawer ($81.5\to92.0\%$) and basket ($58.0\to70.5\%$).

\subsection{Held-out action fidelity}
On P2, retaining o/d at INT8 raises the minimum decoded-action cosine on all four checkpoints: from $0.461$ to $0.850$ for N1.6, $0.802$ to $0.912$ for N1.7, $0.970$ to $0.973$ for N1.5, and $0.847$ to $0.9985$ for \pih{}. Their median maximum absolute coordinate errors fall from $0.074$ to $0.039$ for N1.6 and from $0.099$ to $0.058$ for N1.7. W8A8 remains closer to the float reference on these metrics.

A few observations dominate the mean cosine deficit ($1-\cos$): the worst W4A4 observation contributes $65\%$ for N1.6 and $43\%$ for N1.7. Across 19 arms with P1/P3 measurements (FQ W8A8/W4A4 and ModelOpt SQ/AWQ on four checkpoints, plus three o/d INT8 arms), cosine is at least $0.998$ and its rank correlation with success differences is not significant (Spearman $\rho=0.20$, $p=0.42$). Shared checkpoints and references make this pooled correlation descriptive.

\subsection{Design ablations}
\begin{table}[t]
\centering\scriptsize
\setlength{\tabcolsep}{2.5pt}\renewcommand{\arraystretch}{1.05}
\caption{Offline fidelity and LIBERO ablations. Design comparisons report changes in successes out of 800 (first minus second), with paired task-level $p$ values in parentheses. Calibration controls are separated below; --: not evaluated.}
\label{tab:ablation}
\begin{tabular*}{\columnwidth}{@{\extracolsep{\fill}}llrr@{}}
\toprule
First vs.\ second & Offline & N1.6 & N1.7\\\midrule
Consistent vs.\ mismatched fold & cos $.9997$ vs.\ $.9923$ & -- & --\\
GPTQ vs.\ RTN & deficit $-25\%/{-37}\%$$^{a}$ & $-6$ ($.45$) & $+7$ ($.23$)\\
Butterfly vs.\ composite rot. & $|\Delta\cos|\le2{\times}10^{-4}$ & $+10$ ($.18$) & $-1$ ($.87$)\\
Fold-before vs.\ after & $|\Delta\cos|\le2{\times}10^{-4}$ & $+4$ ($.66$) & $-11$ ($.19$)\\
\textbf{o/d INT8 vs.\ W4A4} & deficit $-65\%$ & $+10$ ($.17$) & $+3$ ($.77$)\\
\midrule
\multicolumn{4}{@{}l}{\emph{Calibration controls (W8A8)}}\\
Per-token vs.\ per-tensor & LLM cos $.998$ vs.\ $.991$ & -- & --\\
ModelOpt per-tensor vs.\ BF16 & -- & $-41.0$\,pp$^{b}$ & $-5.5$\,pp$^{b}$\\
\bottomrule
\end{tabular*}

\vspace{2pt}
\parbox{\columnwidth}{\scriptsize Offline: action-module cosine for the fold comparison, decoded-action fidelity for the remaining design comparisons. Deficit is $1-\cos$; its changes are relative percentages. $^{a}$N1.6/N1.7; other offline entries are N1.6. $^{b}$Percentage-point SR changes: external ModelOpt per-tensor max calibration with no smoothing versus BF16 (800 episodes; McNemar $p<0.001$). This control also changes the calibration and execution pipeline.}
\end{table}
In Table~\ref{tab:ablation}, the external per-tensor control also changes calibration and execution, so it does not isolate scale granularity. The mismatched fold was tested offline only. FWHT and fold-before are kernel-speed choices with offline cosine changes of at most $2\times10^{-4}$. GPTQ and o/d INT8 improve fidelity; none of these paired design comparisons resolves a closed-loop difference over 800 episodes.

\section{Real-Robot Evaluation}
We deploy the Orin engines on single-arm ALOHA and SO-101, using the tasks in Fig.~\ref{fig:real_tasks}. GR00T N1.7 uses one fine-tune per embodiment. Each arm is tested for 20 episodes per task, interleaved within a session with matched object placements.

\begin{table}[t]
\centering\scriptsize
\setlength{\tabcolsep}{4pt}\renewcommand{\arraystretch}{1.05}
\caption{N1.7 on SO-101: SR (\%) over 20 trials per task and arm on Orin. T2--T4 are defined in Fig.~\ref{fig:real_tasks}; Avg.\ is their mean. $^{\ddagger}$Stopped after T4 because jerky motion risked the hardware; n/a: not evaluated.}
\label{tab:robot}
\begin{tabular*}{\columnwidth}{@{\extracolsep{\fill}}lrrrr@{}}
\toprule
Method & T2 & T3 & T4 & Avg.\\\midrule
BF16 PyTorch & 90.0 & \textbf{100.0} & 70.0 & 86.7 \\
TRT BF16 & 90.0 & \textbf{100.0} & 75.0 & 88.3 \\
ModelOpt W8A8 SQ$^{\ddagger}$ & n/a & n/a & 30.0 & n/a \\
ModelOpt W4A16 AWQ & n/a & n/a & 40.0 & n/a \\
\midrule
\textbf{FQ W8A8 (ours)} & \textbf{95.0} & \textbf{100.0} & 70.0 & 88.3 \\
\textbf{FQ W4A4 (ours)} & 85.0 & 95.0 & 65.0 & 81.7 \\
\textbf{FQ W4A4 + o/d INT8 (ours)} & \textbf{95.0} & \textbf{100.0} & \textbf{80.0} & \textbf{91.7} \\
\bottomrule
\end{tabular*}
\end{table}

\begin{table}[t]
\centering\scriptsize
\begin{minipage}[t]{0.47\columnwidth}
\centering
\caption{N1.7 on ALOHA: SR (\%), 20 trials per arm. T1 as in Fig.~\ref{fig:real_tasks}.}
\label{tab:robot_aloha}
\setlength{\tabcolsep}{3pt}
\begin{tabular*}{\linewidth}{@{\extracolsep{\fill}}lr@{}}
\toprule
Method & \shortstack[r]{T1: banana\\in pot, lid}\\\midrule
BF16 PyTorch & 90.0 \\
TRT BF16 & \textbf{95.0} \\
\midrule
\textbf{FQ W8A8 (ours)} & 85.0 \\
\textbf{FQ W4A4 (ours)} & 75.0 \\
\shortstack[l]{\textbf{FQ W4A4 + o/d}\\\textbf{INT8 (ours)}} & \textbf{95.0} \\
\bottomrule
\end{tabular*}
\end{minipage}\hfill
\begin{minipage}[t]{0.47\columnwidth}
\centering
\caption{\pih{} on SO-101 (its own fine-tune): SR (\%), 20 trials per arm. T4 as in Fig.~\ref{fig:real_tasks}.}
\label{tab:robot_pi}
\setlength{\tabcolsep}{3pt}
\begin{tabular*}{\linewidth}{@{\extracolsep{\fill}}lr@{}}
\toprule
Method & \shortstack[r]{T4: blue block\\on red block}\\\midrule
BF16 PyTorch & 80.0 \\
TRT BF16 & \textbf{100.0} \\
\midrule
\textbf{FQ W4A4 (ours)} & 85.0 \\
\bottomrule
\end{tabular*}
\end{minipage}
\end{table}
\textbf{Selective INT8 improves observed robot outcomes.} Tables~\ref{tab:robot} and~\ref{tab:robot_aloha} show an improvement over uniform W4A4 on every N1.7 task. Pooled success is $92.5\%$ for o/d INT8, $80.0\%$ for W4A4, $87.5\%$ for W8A8 and BF16 PyTorch, and $90.0\%$ for float TensorRT. The o/d INT8 gain costs $1$\,ms on Orin; these observed rates do not establish equivalence to W8A8 or BF16.

On logged ALOHA and SO-101 observations, the median action cosine against BF16 is $0.99962$ and $0.99961$ for N1.7 W4A4, respectively; selective INT8 raises these medians to $0.99975$ and $0.99988$.

ModelOpt W8A8 SmoothQuant uses static per-tensor scales and reached $30.0\%$ on T4 versus $70.0\%$ for BF16 and FQ W8A8. Jerky motion risked the hardware, so testing stopped at T4 and this arm is excluded from the aggregate. Its median SO-101 action cosine is $0.99892$ despite poor task performance. FQ W8A8 also uses rotations, precluding attribution solely to per-token scaling. ModelOpt W4A16 AWQ reaches $40.0\%$ on T4, with mean action cosine $0.99980$ and minimum $0.99882$; it is also excluded from the four-task aggregate.

For \pih{} on SO-101 T4 (Table~\ref{tab:robot_pi}), W4A4 success is close to its own BF16 reference despite a median action cosine of $0.98994$, substantially below N1.7's. A fixed cosine threshold therefore does not transfer automatically between checkpoints.

\section{Limitations and Reproducibility}
\label{sec:threats}
\textbf{Architecture dependence.} Consistent folding does not guarantee useful W4A4 accuracy for every VLA. Archived exploratory LIBERO campaigns (800 episodes per arm) give SmolVLA~\cite{smolvla} $27.1\%$ versus BF16's $71.1\%$ ($217$ versus $569$ successes), and Evo-1~\cite{evo1} $81.0\%$ versus $91.6\%$ ($648$ versus $733$). These are configuration-specific failures; calibration differences limit attribution of Evo-1's entire deficit to bit width. We have not established that o/d INT8 restores either architecture.

\textbf{Runtime and hardware.} Projection coverage does not guarantee a speedup. Native INT4 is measured on Ada and Orin; the inspected H100 backend lowers four-bit operands to INT8 arithmetic. H100 success therefore establishes neither native-INT4 latency nor bitwise cross-device equivalence.

\textbf{Statistical and behavioral scope.} Most simulator arms have one campaign over 40 tasks, and matched initial states do not eliminate serving noise. Physical trials number only 20 per task; \pih{} is tested on one task. An exploratory two-sided Fisher exact test on the pooled N1.7 robot counts ($74/80$ versus $64/80$) gives $p=0.037$, but ignores matched placements and task/session dependence. We therefore report the robot gain as an observed difference. P1/P2 each use 32 observations, cosine ignores magnitude, and binary success omits smoothness and contact forces.

\textbf{Attribution and reproducibility.} N1.6/N1.7 deployment comparisons include calibration changes; only preset-matched controls isolate precision. The two o/d types are evaluated together, leaving their individual effects unresolved. One desktop timing run does not quantify run-to-run variability. Reproduction requires each result's checkpoint, calibration samples, transform preset, and evaluation protocol.

\section{Conclusion}
FoldQuantVLA carries consistent activation coordinates from calibration to native low-bit execution in both language backbone and action expert. W4A4 reduces latency beyond compiled float engines on desktop and Orin. Selective o/d INT8 improves held-out fidelity and raises observed N1.7 robot success from $80.0\%$ to $92.5\%$ for $1$\,ms on Orin, though its LIBERO gains remain unresolved and its cost is larger for \pih{}. Deployment requires checkpoint-specific latency and closed-loop evaluation.
\bibliographystyle{IEEEtran}
\bibliography{refs}
\end{document}